\documentclass[10pt, a4paper, copyright, logo]{instadeep}

\newcommand{\lmax}{L_{\max}}
\newcommand{\logsig}{\sigma_{\!\log}}

\newcommand{\Gtil}{\widetilde{G}}

\title{Integral Formulas for Vector Signal Tensor Products}

\correspondingauthor{v.heyraud@instadeep.com}

\author[1]{Valentin Heyraud}
\author[1]{Zachary Weller-Davies}
\author[1]{\\Jules Tilly}

\affil[1]{InstaDeep, 1 Triton Square, London, NW1 3BF, United Kingdom}

\begin{abstract}
We derive integral formulas that simplify the Vector Signal Tensor Product recently introduced by~\citet{xie2026asymptoticallyfastclebschgordantensor}, which generalizes the Gaunt Tensor Product to antisymmetric couplings. In particular, we obtain explicit closed-form expressions for the antisymmetric analogues of the Gaunt coefficients. This enables us to simulate the Clebsch-Gordan tensor product using a single Vector Signal Tensor Product, yielding up to a $9\times$ reduction in the required tensor product evaluations. Our results enable efficient and practical implementations of the Vector Signal Tensor Product, paving the way for applications of this generalization of Gaunt tensor products in $\mathrm{SO}(3)$-equivariant neural networks. Moreover, we discuss how the Gaunt and the Vector Signal Tensor Products allow to control the expressivity-runtime tradeoff associated with the usual Clebsch-Gordan Tensor Products. Finally, we investigate low rank decompositions of the normalizations of the considered tensor products in view of their use in equivariant neural networks.
\end{abstract}

\begin{document}
\maketitle

\section{Introduction}
Clebsch-Gordan Tensor Products (CGTP) are the most prominent workhorses of $\mathrm{SO}(3)$-equivariant neural networks, which have proven effective for a wide range of applications involving geometric data. In such models, feature vectors transform according to irreducible representations of $\mathrm{SO}(3)$, and nonlinear interactions between these must preserve this transformation law. The CGTP provides a standard combination mechanism that respects rotational symmetry. In practice, this operation decomposes the interaction between two features into several admissible coupling paths, each corresponding to a different way in which their representations can combine.

Unfortunately, CGTP operations are costly; they naively admit a $\mathcal{O}(L^{6})$ scaling with the representation order $L$ of the features on which they act. To address this limitation, \citet{luo2024enabling} proposed using Gaunt Tensor Products (GTP) to accelerate tensor products involving feature vectors with multiple irrep types. Their approach relies on an integral formula for the Clebsch-Gordan coefficients which allows to rewrite the tensor product between feature vectors in terms of triple integrals of the corresponding signals on the sphere.

Subsequently, \citet{xie2025the} investigated the computational and representational properties of GTPs. They compared GTP to alternative tensor product constructions and proposed two efficient implementations, the first one evaluating the required triple integrals using a spherical design~\cite{hardin1996mclaren}, and the second, called S2FFT, relying on a spherical design and a $\mathbb{S}^2$ fast Fourier transforms algorithm. As noted by \citet{xie2025the}, however, a key limitation of the GTP is that it cannot reproduce the antisymmetric components of the CGTP, preventing key underlying tensor product operations such as the cross product. Preserving all coupling paths is important for the expressive power of equivariant neural networks, as these paths capture the different symmetry-preserving ways in which features can interact. 

To circumvent this issue,~\citet{xie2026asymptoticallyfastclebschgordantensor} recently introduced Vector Signal Tensor Products  (VSTP). The VSTP generalizes GTP by constructing vector-valued signals on the sphere. The VSTP exhibits an advantageous complexity scaling in the features irreps order over the CGTP and can encompass both the symmetric and anti-symmetric cases. Specifically, a main result of \cite{xie2026asymptoticallyfastclebschgordantensor} is that a collection of VSTPs can be used to simulate CGTPs, meaning that it is capable of producing non-zero outputs for all admissible coupling paths. 

However, the construction proposed in~\cite{xie2026asymptoticallyfastclebschgordantensor} yield a cumbersome implementation. While the authors show that the VSTP can be used to recover CGTP up to a factor, each irrep label admits up to three admissible internal angular-momentum couplings. Consequently, simulating the CGTP requires computing a VSTP for every pair of couplings, leading to as many as $3\times3=9$ VSTP operations in practice.

In this paper, we build on the results of~\cite{xie2026asymptoticallyfastclebschgordantensor} and derive simple closed-form integral expressions that serve as antisymmetric analogues of the integral GTP formula \cite{luo2024enabling}. Our main result, given in Equation~\eqref{eq: main_result_odd}, shows that the antisymmetric components of the CGTP can be encoded as an integral over the sphere, and we obtain explicit closed-form expressions for the corresponding antisymmetric Gaunt coefficients in Appendix~\ref{sec: appendix}. Moreover, via Equation \eqref{eq: combined_tp}, we show that the symmetric and antisymmetric contributions of the GTP can be combined into a single integral representation, yielding a simple and practical recipe to simulate a CGTP with only one VSTP, while retaining the favorable scaling properties of the VSTP framework. We also discuss the advantages and limitations of such integral-based tensor products, identifying the regimes in which they can effectively replace standard CGTP.

\section{Background}

In this section, we briefly summarize the findings from the references \cite{luo2024enabling, xie2025the, xie2026asymptoticallyfastclebschgordantensor} on which we base our results. In the following we assume a basic knowledge of $\mathrm{SO}(3)$-equivariant neural networks and of the representation theory of $\mathrm{SO}(3)$, for which we refer the to Refs.~\cite{unke2024e3x, geiger2022e3nneuclideanneuralnetworks} and references therein.

We work with the basis of real spherical harmonics, as it is the basis commonly used in python libraries for $\mathrm{SO}(3)$-equivariant neural networks \cite{geiger2022e3nneuclideanneuralnetworks, unke2024e3x}. We denote $Y_{lm}$ the real spherical harmonic polynomials and $C^{l_3m_3}_{l_1m_1, l_2m_2}$ the Clebsch-Gordan coefficients in this basis. 

\subsection{Clebsch-Gordan Tensor Products}

The CGTP allows to couple equivariant features of irreps type $(l_1, l_2)$ to produce new equivariant of irrep type $l_3$ provided the triplet $(l_1,l_2,l_3)$ satisfy the triangular condition
\bea
|l_1-l_2|\leq l_3 \leq l_1+ l_2\,.
\eea
For feature vectors $h^{l_1}, h^{l_2}$ it is defined using the Clebsch-Gordan coefficients as
\begin{equation}
(h^{l_1}\otimes h^{l_2})^{l_3m_3} := \sum_{m_1, m_2} C^{l_3 m_3}_{l_1m_1, l_2m_2} h^{l_1m_1}h^{l_2m_2}\,.
\end{equation}
The Clebsch-Gordan coefficients satisfy the symmetry relation~\cite{Varshalovich1988}
\bea
\label{eq:cg-coef-sym}
C^{l_3m_3}_{l_1m_1, l_2m_2} = (-1)^{l_1+l_2-l_3}C^{l_3m_3}_{l_2m_2, l_1m_1}\,,
\eea
from which it appears that the CGTP is anti-symmetric whenever the triplet $l_1 + l_2 - l_3$ is odd and symmetric whenever it is even.

\subsection{Gaunt Tensor Products}

\citet{luo2024enabling} proposed GTPs as fast alternative to CGTP. We discuss the regime in which GTPs can be beneficial in Section~\ref{sec: expressivity}. GTPs rely on the following formula relating Gaunt coefficients, defined as a triple integral of spherical harmonic polynomials, to the Clebsch-Gordan coefficients
\bea
\label{eq: gcoef-integral}
G^{l_3 m_3}_{l_1 m_1, l_2 m_2} :&= \int_{\sphere} Y_{l_1m_1}(\hat{r})Y_{l_2m_2}(\hat{r})Y_{l_3m_3}(\hat{r}) \mathrm{d}\mu_{\sphere}(\hat{r})\\
&= \tilde{G}^{l_3}_{l_1, l_2} C^{l_3 m_3}_{l_1 m_1, l_2 m_2}\,.
\eea

Here $\hat{r}$ is the unit vector of spherical coordinates ${(\theta, \phi)\in[0,\pi]\times[0,2\pi)}$ and ${\mathrm{d}\mu_{\mathbb{S}^2}(\hat{r}) = \sin(\theta)\mathrm{d}\theta\mathrm{d}\phi}$ is the uniform measure on the sphere $\sphere$. To simplify notations, in the following we sometimes omit the $\hat{r}$ argument in integrals.

GTPs are defined by replacing the Clebsch-Gordan coefficients by the Gaunt coefficients
\begin{equation}
(h^{l_1}\otimes_{\mathrm{Gaunt}} h^{l_2})^{l_3m_3} = \sum_{m_1, m_2} G^{l_3 m_3}_{l_1m_1, l_2m_2} h^{l_1m_1}h^{l_2m_2}\,,
\end{equation}
such that
\bea
(h^{l_1} \otimes_{\mathrm{Gaunt}} h^{l_2})^{l_3m_3} &= \tilde{G}^{l_3}_{l_1,l_2}(h^{l_1} \otimes h^{l_2})^{l_3 m_3}\,.
\eea

One can rely on the integral formula in Equation \ref{eq: gcoef-integral} to re-write the GTP as
\bea
\left(h^{l_1} \otimes_{\mathrm{Gaunt}} h^{l_2}\right)^{l_3m_3}= \sum_{m_1, m_2}h^{l_1m_1}h^{l_2m_2}\int_{\mathbb{S}^{2}} Y_{l_1m_1}Y_{l_2m_2}Y_{l_3m_3} \mathrm{d}\mu_{\mathbb{S}^2}.
\eea
This expression can be further factorised by introducing the signal on the sphere associated with the features $h^{l_1}, h^{l_2}$. In fact, an equivariant feature $h^{l}$ yields a signal of the sphere of the form
\bea
F_{h^{l}}:\mathbb{S}^{2}&\to\mathbb{R}\\
\hat{r}&\mapsto \langle h^{l}, Y_{l}(\hat{r})\rangle\,
\eea
where $\langle h^{l}, Y_{l}(\hat{r})\rangle := \sum_{m=-l}^{m} h^{lm}Y_{lm}(\hat{r})$. In this notation, the GTP can be rewritten more compactly as 
\bea \label{eq: gtp-vectorvaluedintegral}
(h^{l_1} \otimes_{\mathrm{Gaunt}}h^{l_2}) ^{l_3m_3}=\int_{\mathbb{S}^{2}} F_{h^{l_1}}F_{h^{l_2}} Y_{l_3m_3}\mathrm{d}\mu_{\mathbb{S}^2}\,.
\eea
This integral form of the GTP allows to factorize sums of tensor products corresponding to different irreps paths $(l_1, l_2)\to l_3$, which can lead to a significant speed-up. We discuss the benefits and limitations of this approach in Section~\ref{sec: expressivity}.

\citet{luo2024enabling} originally proposed to use a method based on Fast Fourier Transform to estimate the triple integral defining the GTP. Instead, \citet{xie2025the} proposed to use a spherical $t$-cubature~\cite{mclaren1963optimal, hardin1996mclaren, Womersley2018,PRICE2024113109}, namely a finite set of points $\left\{\hat{r}_i\right\}_{1\leq i\leq N}\in\mathbb{S}^{2}$ on the sphere with associated weights ${\left\{w_i\right\}_{1\leq i\leq N}\in\mathbb{R}}$ such that for any polynomial $P:\mathbb{S}^{2}\mapsto \mathbb{C}$ of degree $t$,
\bea
\sum_{i=1}^{N} w_i P(\hat{r}_i) = \int_{\mathbb{S}^{2}} P(\hat{r})\mathrm{d}\mu_{\mathbb{S}^2}(\hat{r})\,.
\eea
Remarking that the spherical harmonics $Y_{lm}$ are uniform polynomials of degree $l$, it suffices to use a spherical $(l_1+l_2+l_3)$-cubature to evaluate the triple integral. Given that a spherical $L$-cubature contains roughly $\mathcal{O}(L^2)$ points~\cite{mclaren1963optimal}, \citet{xie2025the} estimate the runtime of the GTP with this evaluation method to be $\mathcal{O}(L^2\log(L))$ for a single triplet of irreps $l_1=l_2=l_3=L$. The authors also propose an evaluation method based on the fast spherical transform algorithm of \citet{healy2003ffts} with a similar runtime for a single triplet but yielding an improved asymptotic runtime in the case of multiple input and output irreps which we discuss in Section~\ref{sec: expressivity}.

From the definition of the Gaunt coefficients and the symmetry property of the Clebsch-Gordan coefficient of Eq.~\eqref{eq:cg-coef-sym}, one clearly sees that $\tilde{G}^{l_3}_{l_1, l_2} =0$ whenever $l_1+l_2-l_3$ is odd. Consequently, the GTP fails to reproduce the CGTP for anti-symmetric cases. 

\subsection{Vector Signal Tensor Products}

\citet{xie2026asymptoticallyfastclebschgordantensor} recently proposed a generalization of the GTP approach encompassing both the symmetric and the anti-symmetric cases of the CGTP. They suggest to replace the signals on the sphere used in GTP by irreducible tensor-valued signals. To this end they propose to use the so-called Tensor Spherical Harmonics (TSH)
\bea
Y^{l s}_{j m_j}:\mathbb{S}^2\to \mathbb{R}^{2s+1}
\eea
which entries are defined as
\bea
\left(Y^{l s}_{j m_j}\right)_{m_s} := \sum_{m_l} C^{j, m_j}_{l, m_l, s, m_s} Y_{lm_l}\,.
\eea
TSH are indexed by the integers $(j,l,s,m)$ where the Clebsch-Gordan selection rules impose $|j-s|\leq l\leq j+s$ and $-j\leq m_j\leq j$. These tensors are used in physics, where they correspond to a coupling of a spin angular momentum $s$ with an orbital momentum $l$ to obtain a total angular momentum $j$ \cite{Varshalovich1988}. When $s=1$, TSH are refered to as Vector Spherical Harmonics (VSH). Note that for a given value of $j$ there are only three VSH corresponding to the only values of $l$ allowed by the selection rule, namely $j-1, j$ and $j+1$. VSH can be decomposed in terms involving the gradients of the associated spherical harmonics, which we leverage to demonstrate our main results presented in Section~\ref{sec: main-results}, see Appendix~\ref{sec: appendix}.

The authors suggest to replace the usual equivariant feature $h^{l}$ by tensor features $h^{(j,l)}$, which contraction with THS yields irreps-valued signal on the sphere. They then define a generalization of GTP for these irreps-valued tensor, which they call Irrep Signal Tensor Product (ISTP). The authors then derive a formula to expand the tensor product of TSH generalizing existing formula for $s=1$. This allows them to obtain a generalization of the integral formula of GTP to ISTP, where point-wise product of the signals are replaced by tensor products of the corresponding irrep-valued signals. They further show that to obtain all possible tensor products interactions, it suffices to consider ISTP restricted to VSH, which they call Vector Signal Tensor Product (VSTP). Crucially, they prove that VSTP are enough to simulate all the CGTP corresponding to different values of $j$ for the inputs and outputs. In particular, they show that for a CGTP corresponding to an irrep path $(j_1, j_2)\to j_3$, at least one of the $3\times3=9$ interactions between the input's VSH yields non-zero output of type $j_3$.

The VSTP introduced in Ref.\cite{xie2026asymptoticallyfastclebschgordantensor} inherits the same efficient asymptotic runtime scaling in $L$ as the GTP and it encompasses both the symmetric and anti-symmetric cases of the CGTP. However, practical applications of VSTPs are hindered by their cumbersome implementation. In fact, VSTPs replace standard spherical-harmonics-based features with VSH-derived equivariant features, and they necessitate computing nine cross products between the irreps-valued signals on the sphere associated with each input. This large number of operations might also harm the potential speed-up of VSTP in the low $L$ regime, which is the one mostly used in current architecture.

\subsection{Other related works}

Other methods have been proposed to accelerate tensor product in a variety of specific setups. In particular, ~\citet{pmlr-v202-passaro23a} proposed a method to accelerate tensor products between generic equivariant features and a vector of spherical harmonic features, relying on a well-chosen rotation of the feature vectors. \citet{li2025eformer} considered a similar setup in the context of equivariant graph neural networks. Their proposed method replaces edge-based tensor products with node-based operations by relying on a binomial expansion of the spherical harmonic feature vector. Unlike for the GTP and VSTP, these methods cannot by directly be applied to tensor products between generic feature vectors, and we do not discuss them in this work.

Shortly after the initial submission of this work, \citet{bochkarev2026} independently proposed a similar method to generalize GTP to anti-symmetric cases.
\section{Integral Formula for Anti-Symmetric Tensor Products}
\label{sec: main-results}

In this section we present our main findings, namely integral formulas involving gradients of the spherical harmonics that allow us to simplify VSTP and to propose an easy-to-implement extension of GTPs to anti-symmetric cases of CGTP. Our first result is the following integral equation.

\begin{theorem}
\label{thm: antisym-integral-formula}
Let $(l_1,l_2,l_3)$ be a triplet satisfying the anti-symmetric condition, i.e. $l_1+l_2-l_3$ is odd. Then there exists $\tilde{V}^{l_3}_{l_1, l_2}\neq 0$ such that
\bea \label{eq: main_result_odd}
    \int_{\mathbb{S}^{2}} \left(\left(\nabla Y_{l_1m_1}\times \nabla Y_{l_2m_2}\right)\cdot \hat{r}\right)Y_{l_3m_3}\mathrm{d}\mu_{\mathbb{S}^{2}}
    = \tilde{V}^{l_3}_{l_1,l_2}C^{l_3m_3}_{l_1m_1, l_2 m_2},
\eea
\end{theorem}
\begin{proof} \renewcommand\qedsymbol{}
See Section~\ref{sec: proof-thm-1} of the appendix.
\end{proof}

\begin{figure*}[t]
  \centering
  \includegraphics[width=\textwidth]{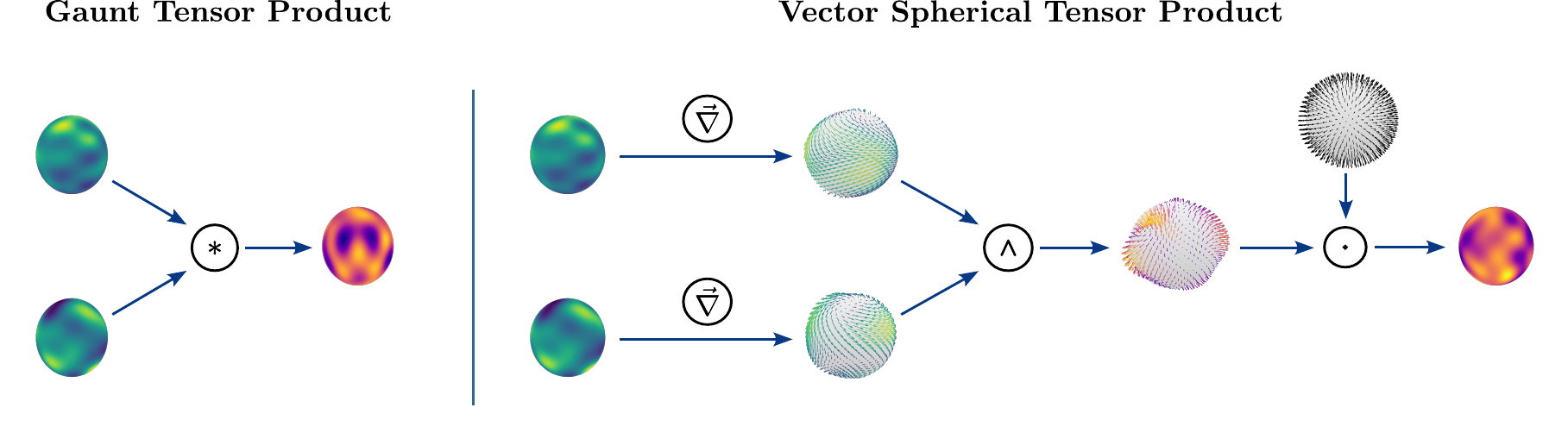}

  \caption{%
    Sketch of the Gaunt Tensor Product and our simplified Vector Signal Tensor Products. The inputs irrep features are first transformed into scalar signals on the sphere using spherical harmonics. Gaunt Tensor Products operates the pointwise product of the signals and the resulting signal is turned back into an irrep features by projecting on spherical harmonics. To obtain the anti-symmetric counter part, VSTP takes the cross product of the signals gradients and project the result pointwise on the radial vectors. The resulting scalar signal is then projected back to irrep features as for Gaunt Tensor Products.
  }
  \label{fig:vstp_gtp_sketch}
\end{figure*}

This formula yields an integral expression of the CGTP in the anti-symmetric cases. 
Since it involves the cross product of spherical harmonics gradients, which can be expressed in terms of VSH, this expression corresponds to a specific interaction involved in the VSTP proposed by~\citet{xie2026asymptoticallyfastclebschgordantensor}. Given this connection, we propose redefining the VSTP as
\bea
\int_{\mathbb{S}^{2}} \left((\nabla \langle h^{l_1}, Y_{l_1}\rangle\times \nabla \langle h^{l_2}, Y_{l_2}\rangle)\cdot \hat{r}\right)Y_{l_3m_3}\mathrm{d}\mu_{\mathbb{S}^{2}}
:= \left(h^{l_1}\otimes_{\mathrm{VSTP}} h^{l_2}\right)^{l_3m_3}\,.
\eea
The previous theorem yields
\bea
\left(h^{l_1}\otimes_{\mathrm{VSTP}} h^{l_2}\right)^{l_3m_3} = \tilde{V}^{l_3}_{l_1,l_2}\left(h^{l_1}\otimes h^{l_2}\right)^{l_3m_3}\,.
\eea
Note that the above formula hold true for all triplets $(l_1, l_2, l_3)$ provided one set $\tilde{V}^{l_3}_{l_1, l_2}=0$ whenever the triplet corresponds to a symmetric CGTP.

Using the previous result, we derive the following integral formula encompassing both the symmetric and the anti-symmetric cases.

\begin{theorem}
\label{thm: total-integral-formula}
For all triplets $(l_1, l_2, l_3)$, there exists $\Gamma^{l_3}_{l_1, l_2}\neq0$ such that
\bea\label{eq: combined_tp}
 \left(h^{l_1}\otimes h^{l_2}\right)^{l_3m_3} = \Gamma^{l_3}_{l_1, l_2}\int_{\mathbb{S}^{2}} &\left(\langle h^{l_1}, Y_{l_1}\rangle\hat{r} + \hat{r}\times \nabla \langle h^{l_1}, Y_{l_1}\rangle\right)\cdot \left(\langle h^{l_2}, Y_{l_2}\rangle\hat{r}+\nabla \langle h^{l_2}, Y_{l_2}\rangle\right)Y_{l_3 m_3}\mathrm{d}\mu_{\mathbb{S}^{2}}\,,
\eea
\end{theorem}
\begin{proof}
Using the identity of the mixed product under cyclic permutation
\bea
(\vec{a}\times\vec{b})\cdot \vec{c} = (\vec{c}\times \vec{a})\cdot \vec{b}
\eea
one can rewrite the VSTP of $h^{l_1}, h^{l_2}$ as
\bea\label{eq: vstp-rewrite}
\int_{\mathbb{S}^{2}} \left((\hat{r}\times \nabla \langle h^{l_1}, Y_{l_1}\rangle)\cdot (\nabla \langle h^{l_2}, Y_{l_2}\rangle)\right)Y_{l_3m_3}\mathrm{d}\mu_{\mathbb{S}^{2}}\,.
\eea 
To encompass both the symmetric and the anti-symmetric case in a single formula, one can simply sum the GTP and VSTP, which gives
\bea
 \left(h^{l_1}\otimes_{\mathrm{Gaunt}} h^{l_2}\right)^{l_3m_3}+  \left(h^{l_1}\otimes_{\mathrm{VSTP}} h^{l_2}\right)^{l_3m_3} = \left(\tilde{G}^{l_3}_{l_1, l_2}+\tilde{V}^{l_3}_{l_1, l_2}\right) \left(h^{l_1}\otimes h^{l_2}\right)^{l_3m_3}\,.
\eea
As the coefficients $\tilde{G}^{l_3}_{l_1, l_2}$ and $\tilde{V}^{l_3}_{l_1, l_2}$ are non-zero on distinct triplets, their sum is always non-zero so we can define
\bea
\Gamma^{l_3}_{l_1, l_2} := \left(\tilde{G}^{l_3}_{l_1, l_2}+\tilde{V}^{l_3}_{l_1, l_2}\right)^{-1},
\eea
and we have $\Gamma^{l_3}_{l_1, l_2}\neq 0$ for all $(l_1, l_2, l_3)$. Using the facts that
\bea
(\hat{r}Y_{l_1m_1})\cdot(\nabla Y_{l_2m_2}) &= 0\,,\\
(\hat{r}Y_{l_1m_1})\cdot(\hat{r}\times\nabla Y_{l_2m_2}) &= 0\,,\\
(\hat{r}Y_{l_1m_1})\cdot(\hat{r}Y_{l_2m_2}) &= Y_{l_1m_1}Y_{l_2m_2}
\eea
and expanding the terms in the theorem's integral, we recognize the integral expressions of the GTP (Eq. \ref{eq: gtp-vectorvaluedintegral})  and the VSTP (Eq. \ref{eq: vstp-rewrite}), which proves the result.
\end{proof}

These results simplify the VSTP proposed by \citet{xie2026asymptoticallyfastclebschgordantensor} by reducing the number of VSTP interactions to consider when simulating a given CGTP from up to $9$ possibilities to a single one. Moreover it allows for a clear and simple implementation involving only standard irrep features $h^{l}\in\mathbb{R}^{2l+1}$ instead of the tensor-valued features used in the VSTP proposed by~\citet{xie2026asymptoticallyfastclebschgordantensor}. In fact, one can straightforwardly adapt the spherical cubature implementation introduced in Ref.~\cite{xie2025the} to the previous formulas. 

Our result provide an intuitive interpretation of the VSTP in terms of signals on the sphere, which is represented in Figure~\ref{fig:vstp_gtp_sketch}. The GTP let the features interact by multiplying their corresponding signals pointwise on the sphere. As this operation is symmetric, so are the GTP. Instead, the VSTP let signals interact through their Poisson bracket \cite{lee2003smooth} given by the cross product of the gradients projected along the radial vector. Similarly to the GTP, the VSTP inherits the anti-symmetry of the cross product and it allows to recover the anti-symmetric part of the CGTP.

\section{Expressivity-Runtime Tradeoff of Integral-Based Tensor Products}
\label{sec: expressivity}

In this section we discuss the expressivity-runtime tradeoff of GTP and VSTP, building on the insights of Refs. \cite{xie2025the, xie2026asymptoticallyfastclebschgordantensor}. We focus on GTP to simplify the presentation, although the discussion holds true for all similar integral-based tensor products, including VSTP and our approach. 

In the context of $\mathrm{SO}(3)$-equivariant neural networks, tensor products are used to build equivariant bilinear layers with learnable parameters. \citet{xie2025the} rely on this observation to define a measure of expressivity for the layers built on different types of tensor products. They then compare this measure of expressivity to the corresponding runtime. The same authors extend their analysis to VSTP and detail their findings depending on the number of irreps in the input and output in Ref.~\cite{xie2026asymptoticallyfastclebschgordantensor}.

Here we only discuss the case of tensor products between input features with multiple irrep types, as it is the one for which the integral formulations of tensor products find their relevance. This setup is refereed to as a multiple-input-multiple-output (MIMO) setup in Refs.~\cite{xie2025the, xie2026asymptoticallyfastclebschgordantensor}.

Suppose one wants to build an equivariant bilinear layer acting on inputs
\bea
h^{1} = \bigoplus_{l_1=0}^{L_1} h^{l_1}\,,\quad h^{2} = \bigoplus_{l_2=0}^{L_2} h^{l_2}
\eea
and producing an output feature
\bea
h^{3} = \bigoplus_{l_3=0}^{L_3}h^{l_3}\,.
\eea
Due to the universal property of the CGTP, the most general form of such a layer is given by
\bea
h^{3} &= \bigoplus_{l_3=0}^{L_3}\left( \sum_{l_1=0}^{L_1}\sum_{l_2=0}^{L_2}w^{l_3}_{l_1, l_2}(h^{l_1} \otimes h^{l_2})^{l_3}\right)\,,
\eea
where $w^{l_3}_{l_1, l_2}$ are learnable weights parameterizing the layer. To simplify the discussion, we suppose that $L_1=L_2=L_3=L$. As discussed in Refs.~\cite{xie2025the, xie2026asymptoticallyfastclebschgordantensor}, the runtime associated with this operation scales as $\mathcal{O}(L^{5})$ when leveraging the sparsity of the Clebsch-Gordan coefficients.

An important aspect of the previous layer is that the CGTP allows to have different weights $w^{l_3}_{l_1, l_2}$ for each possible irreps path $(l_1,l_2)\to l_3$. Integral representations of the tensor product allows to reduce the complexity of the previous layer upon assuming the learnable weights can be factorized, i.e.
\bea
w^{l_3}_{l_1, l_2} = w^{l_3}w_{l_1}w_{l_2}\,.
\eea
In this case, replacing the CGTP by a GTP and using the corresponding integral formula yields
\bea
h^{l_3} &= \sum_{l_1}^{L}\sum_{l_2}^{L}w^{l_3}_{l_1, l_2}(h^{l_1} \otimes_{\mathrm{Gaunt}} h^{l_2})^{l_3}\\
&=\sum_{l_1=0}^{L}\sum_{l_2=0}^{L}w^{l_3} w_{l_1} w_{l_2}\int_{\mathbb{S}^{2}}F_{h^{l_1}}F_{h^{l_2}}Y_{l_3}\mathrm{d}\mu_{\mathbb{S}^2}\\
&= w^{l_3}\int_{\mathbb{S}^{2}}\left(\sum_{l_1=0}^{L} w_{l_1}F_{h^{l_1}}\right)\left(\sum_{l_2=0}^{L} w_{l_2}F_{h^{l_2}}\right)Y_{l_3}\mathrm{d}\mu_{\mathbb{S}^2}\,.
\eea

This formula has the advantage that the signals corresponding to the different irreps order of the different input features are factorized and can be summed before the product and the integration are performed. As shown by \citet{xie2025the}, this significantly reduce the cost of the corresponding layer to $\mathcal{O}(L^3)$ for a simple spherical design evaluation and to $\mathcal{O}(L^2\log^2(L))$ for the S2FFT evaluation proposed in Ref.~\cite{xie2025the}. Note that the authors mention that their S2FFT evaluation method is more efficient only asymptotically in $L$ and that it doesn't result in a practical speed-up below $L\sim 1000$, which is far beyond the values of $L$ used in current architectures.

The quadratic asymptotic speed-up in $L$ of GTP relies on the weights factorization, which has a strong impact on the expressivity of the corresponding layer. \citet{xie2025the} define the expressivity of the layers above as the dimension of the manifold of the associated bilinear maps. In the case of CGTP, the weights do not need to be factorized so this dimension is of the order of the number of path allowed by the Clebsch-Gordan selection rules, which scales as $\mathcal{O}(L^3)$. On the other hand, the GTP advantageous runtime assumes the factorization of the weights, which yields a corresponding dimension scaling in $\mathcal{O}(L)$. Consequently, \citet{xie2025the} conclude that the CGTP and the GTP layers share the same $\mathcal{O}(L^2)$ asymptotic runtime-expressivity ratio.

Although the analysis of \citet{xie2025the} offers a clear insight on the asymptotic expressivity-runtime tradeoff, we argue that the GTP can still prove useful when the learnable weight $w^{l_3}_{l_1, l_2}$ admits a rank $R$ decomposition~\cite{kolda2009tensor}
\bea\label{eq: low_rank_decom}
w^{l_3}_{l_1, l_2} = \sum_{k=1}^{R} w^{(k) l_3}w^{(k)}_{l_1}w^{(k)}_{l_2}\,.
\eea
For such weights the GTP layer yields a runtime scaling as $\mathcal{O}(RL^3)$, whereas the corresponding CGTP layer does not leverage the weights decomposition and yields the same $\mathcal{O}(L^5)$ runtime. Thus it appears that the expressivity measure presented by \citet{xie2025the} is a worst-case scenario measure which assumes the learnable weights have full rank $R=\mathcal{O}(L^2)$.

This analysis shows that integral-based tensor products as the ones considered in this paper allow to control the runtime-expressivity tradeoff by enabling shorter runtime for bilinear layers which parameters $w^{l_3}_{l_1, l_2}$ admit a low rank decomposition. It would be interesting to investigate this tradeoff numerically in future work.

\section{Normalization}
When the antisymmetric GTP in Equation \eqref{eq: main_result_odd}, or the full tensor product in Equation \eqref{eq: combined_tp}, are used as learnable layers in an equivariant neural network, it is important that the coupling coefficients are properly normalized. One typically wishes the magnitude of the resulting features to be comparable to those produced by the standard CGTP, so that layer initialization remains well-conditioned and does not introduce artificial scale differences across angular momentum channels.

While the Clebsch-Gordan tensor $C^{l_3 m_3}_{l_1 m_1\,l_2 m_2}$ is
orthonormal by construction, the symmetric and antisymmetric coefficients $T_{l_1, l_2, l_3} \in \{ \tilde{G}_{l_1, l_2, l_3}, \tilde{V}_{l_1, l_2, l_3}\}$ introduce an angular-momentum-dependent scale. Consequently, if the tensor product is used directly in a neural network layer, different $(l_1,l_2,l_3)$ channels may produce outputs with substantially different magnitudes.

A natural strategy is therefore to explicitly normalize the tensor
product by multiplying each channel by the inverse factor
$T_{l_1 l_2 l_3}^{-1}$.  However, in an efficient implementation
of GTP \cite{luo2024enabling, xie2025the}, as described in Section \ref{sec: expressivity}, the spherical integral is evaluated using a factorized representation that separates the dependence on $l_1$, $l_2$, and $l_3$.  Introducing a fully dense normalization tensor would destroy this factorization and therefore remove the computational advantage. 

To preserve the factorized structure, in Appendix \ref{app:lowrank}, we instead seek approximate low-rank decompositions~\cite{kolda2009tensor} of the inverse coupling tensors of the form
\begin{equation}
\label{eq: decomp}
  (T_{l_1 l_2 l_3})^{-1}
  \;\approx\;
  \sum_{r=1}^{R}
  a_{l_1,r}\,
  b_{l_2,r}\,
  c_{l_3,r},
\end{equation}
where the vectors $a_{l_1,r}$, $b_{l_2,r}$ and $c_{l_3,r}$ depend only
on a single angular momentum index.  Substituting this form into the
tensor product preserves the separability of the computation and
allows the normalization to be implemented with negligible additional
cost (a factor of $R$, which can be evaluated in parallel). As shown in Appendix~\ref{app:lowrank}, we find empirically that $(T_{l_1, l_2}^{l_3})^{-1}$ is intrinsically low rank, at least if we only care about order of magnitude accuracy . In particular, a rank-$2$ decomposition of $\tilde{V}_{l_1 l_2 l_3}^{-1}$ reproduces normalization of the full tensor over a wide range of angular momenta that were tested (up to $L_{max} < 20$). By contrast, a rank-$1$ approximation of $\tilde{V}$ fails qualitatively and cannot reproduce the structure of the tensor. Conversely, we find that $\tilde{G}^{-1}_{l_1, l_2, l_3}$ is well approximated by a rank-$1$ tensor. The difference between the approximate ranks of the antisymmetric and symmetric Gaunt coefficients likely arises from the two distinct coupling pathways $l_3 \pm 1$ appearing in Equation \eqref{eq: lambda}, which introduces two dominant modes in the coupling tensor.

In practice, this low-rank representation provides a way to initialize Gaunt tensor-product layers with a scale comparable to CGTP while preserving the factorization structure that enables efficient implementations. For example, using Equation \eqref{eq: low_rank_decom}, a simple way to initialize a neural network with the same scale as the Clebsch-Gordan coefficients would be to choose the initialization 
    \bea
w^{l_3}_{l_1, l_2} = \sum_{k=1}^{R} a_{k, l_1} b_{k, l_2} c_{k, l_3} w^{(k) l_3}w^{(k)}_{l_1}w^{(k)}_{l_2}\,.
\eea
where $a_{k, l_1} b_{k, l_2} c_{k, l_3}$ are the low rank matrices found in Equation \eqref{eq: decomp} and $w^{(k) l_3}w^{(k)}_{l_1}w^{(k)}_{l_2}$ are initialized around the identity.

\section{Conclusion}

In this work, we have arrived at an integral representation (Equation \eqref{eq: main_result_odd}) that generalizes the Gaunt coefficient formula for Clebsch-Gordan coefficients to the antisymmetric case. From this result we obtained Equation \eqref{eq: combined_tp}, a single universal integral expression that encompasses both the antisymmetric and symmetric components of the Clebsch-Gordan coefficients. Such integral formulas have well known and efficient implementation schemes \cite{luo2024enabling, xie2025the} that can be directly applied to this work. Our formulation further shows that a single VSTP is enough to simulate the CGTP  \cite{xie2026asymptoticallyfastclebschgordantensor}. Crucially, the integral formula uses only standard irrep features $h^{l}$, instead of the more general tensor-valued features, which ultimately reduces the number of VSTPs required by a factor of up to 9.

Beyond computational efficiency, having analytic expressions for the even and odd-parity Gaunt coefficients is particularly important for applications to equivariant neural networks. The expressions derived in Appendix \ref{sec: proof-thm-1} allow one to directly study the magnitude and scaling behavior of the coupling coefficients across the angular momentum channels. Using the explicit form of the coefficients derived in this work, we showed in Appendix \ref{app:lowrank} that the resulting normalization tensors admit accurate low-rank approximations. This enables efficient normalization schemes that preserve the factorized structure of the tensor product and therefore maintain the computation advantages of the VSTP construction. 

Together, these results provide a practical recipe for incorporating integral-based tensor products into $\mathrm{SO}(3)$-equivariant neural networks and for controlling the speed-expressivity tradeoff in equivariant architectures. A natural direction for future work is to evaluate the proposed implementation on concrete learning tasks, for
example within machine-learning interatomic potential (MLIP) models~\cite{batatia2022mace, wood2026umafamilyuniversalmodels, brunken2025mlip}, where efficient and well-conditioned equivariant tensor products are critical for scalability and performance.

\bibliography{main}
\newpage
\appendix
\onecolumn
\section{Appendix}\label{sec: appendix}

Unless otherwise stated, in this appendix we use complex spherical harmonics which we denote $Y^{l}_{m}$, and we write $C^{l_3m_3}_{l_1m1; l_2m_2}$ the Clebsch-Gordan coefficients in this basis. The spin components of the vector spherical harmonics are defined in the spherical coordinates basis. We follow the notations of Ref.~\cite{Varshalovich1988} and denote respectively $\cdot$ the Euclidian dot product and $\times$ the cross product in the Cartesian basis. We show explicitly how to transform our results to the real spherical harmonics basis used in the main text in Section \ref{sec: transformation}.

\subsection{Vector Harmonics}\label{sec: vector_harmonics}
In this Section, we introduce the relevant background on vector spherical harmonics used in the appendix. For a more detailed overview coverage of vector harmonics we recommend \cite{Varshalovich1988}. 

As in \cite{Varshalovich1988, xie2026asymptoticallyfastclebschgordantensor}, we define the tensor (spin-$s$) spherical harmonics by coupling
orbital angular momentum $l$ with spin $s$
\begin{equation}\label{eq: tensor_harmonic}
\big(Y^{l,s}_{j m}(\hat{r})\big)_{m_s}
=
\sum_{m_l}
C^{j m}_{l m_l,\, s m_s}\,
Y^l_{m_l}(\hat{r})\,.
\end{equation}
The vector harmonics arise from the $s=1$ components, and for notational convenience, we will write the $s=1$ vector harmonic with components $m_s$ as  $(\vec{Y}_{j,l,m})_{m_s}: = (Y_{lm}^{j 1})_{m_s}$. Note that for the vector harmonics, the selection rules dictate that $j \in \{l-1, l, l+1\}, j \geq 0$. The tensor harmonics obey the orthonormality condition 
\begin{equation}\label{eq: orthonormality}
     \int_{\mathbb{S}^2}
Y^{l,s}_{j m}(\hat{r}) \cdot 
Y^{l',s'}_{j' m'}(\hat{r})^* \  d\mu_{\mathbb{S}^2}(\hat{r}) =   \int_{\mathbb{S}^2}
\sum_{m_s}
\big(Y^{l,s}_{j m}(\hat{r})\big)_{m_s}
\big(Y^{l',s'}_{j' m'}(\hat{r})\big)^{*}_{m_s}  d\mu_{\mathbb{S}^2}(\hat{r})
= 
\delta_{ll'}\,
\delta_{ss'}\,
\delta_{j j'}\,
\delta_{m m'}.
\end{equation}
The conjugation of the harmonics with the dot product arises because we use complex spherical coordinates, and so the correct inner product is the Hermitian inner product \cite{Varshalovich1988}.

For our purposes, it is useful to relate the vector harmonics to the spherical harmonics through application of standard differential operators. In particular, the $j=l$ component of the vector harmonics is known to arise from application of the angular momentum operator $\hat{L} = -i (\hat{r} \times \nabla)$ to the spherical harmonics \cite{Varshalovich1988} 
\begin{equation}\label{eq: ang_mom}
\hat{L} Y^{l}_{m}
=
\sqrt{l(l+1)}\,\vec{Y}_{l,l,m}.
\end{equation}

Similarly, the $j = l \pm 1$ components of the vector harmonics have known relations in terms of the gradient of the spherical harmonics \cite{Jackson, Varshalovich1988}
\begin{equation}\label{eq: higher_j_def}
\begin{split}
\vec{Y}_{l-1,l,m}
&=
-{\sqrt{\frac{2l+1}{l}}}
\left[
l\,\hat{r} - r\,\nabla
\right]
Y^{l}_{m},
\qquad l \neq 0\,,
\\
\vec{Y}_{l+1,l,m}
&=
-{\sqrt{\frac{2l+1}{l+1}}}
\left[
(l+1)\,\hat{r} - r\,\nabla
\right]
Y^{l}_{m}\,,
\end{split}
\end{equation}
where $r = ||\hat{r}||$. 

\subsection{Proof of Theorem 1}\label{sec: proof-thm-1}
In this Section, we prove our main result in Equation \eqref{eq: main_result_odd}. We first show that for complex spherical harmonics
\begin{equation}\label{eq: as_gaunt}
    \int_{\mathbb{S}^{2}}(\nabla Y^{l_1}_{m_1}\times \nabla Y^{l_2}_{m_2})\cdot (\hat{r} Y^{l_3}_{m_3})^* \mathrm{d}\mu_{\mathbb{S}^{2}}(\hat{r}) = \Lambda_{l_1 l_2 l_3}C^{l_3m_3}_{l_1m_1, l_2 m_2},
\end{equation}

where $\Lambda_{l_1, l_2, l_3} \neq 0$ is non-zero whenever a $l_1, l_2, l_3$ path is allowed by the Clebsch-Gordan coefficients and ${l_1 + l_2 - l_3 = 2k+1}$. The proof follows from a number of useful tensor identities that can be found in \cite{Varshalovich1988, xie2026asymptoticallyfastclebschgordantensor}. We derive Equation \eqref{eq: as_gaunt} using the complex harmonics in a spherical coordinate system, and hence the Clebsch-Gordan coefficients should be understood in this basis according to the Condon–Shortley convention. In Section \ref{sec: transformation}, we transform this equation to real harmonics in Cartesian coordinates and relate it to the rotated Clebsch-Gordan coefficients.

We first relate Equation \eqref{eq: as_gaunt} to the vector spherical harmonics using the identities of Section \ref{sec: vector_harmonics}. For $\vec{a}, \vec{b} \perp \hat{r}$, we have that 
\begin{equation}
    (\hat{r}\times \vec{a})\times(\hat{r}\times\vec{b}) = ((\vec{a}\times\vec{b}).\hat{r})\hat{r}.
\end{equation}
Hence, for $||\hat{r}||=1$ we get
\begin{equation}
    \begin{aligned}
(\nabla Y^{l_1}_{m_1}\times \nabla Y^{l_2}_{m_2}) \cdot \hat{r} &= (\hat{r}\times \nabla Y^{l_1}_{m_1})\times(\hat{r}\times \nabla Y^{l_2}_{m_2}) \cdot \hat{r}\\
&=-(\hat{L}Y^{l_1}_{m_1})\times(\hat{L}Y^{l_2}_{m_2}) \cdot \hat{r}\\
&=-\sqrt{l_1 l_2 (l_1+1)(l_2+1)}\left(\vec{Y}_{l_1, l_1, m}\times \vec{Y}_{l_2, l_2, m_2}\right) \cdot \hat{r},
\end{aligned}
\end{equation}
where we use Equation \eqref{eq: ang_mom} between the second and the third line. It therefore remains to evaluate the integral
\begin{equation}\label{eq: as_gaunt_simpl}
    \int_{\mathbb{S}^{2}} \left(\vec{Y}_{l_1, l_2,m}\times \vec{Y}_{l_2,l_2,m_2}\right) \cdot \left(\hat{r} Y^{l_3}_{m_3}\right)^*\mathrm{d}\mu_{\mathbb{S}^{2}}(\hat{r}).
\end{equation}

To proceed, we can use the identity introduced in Equation 101, page 222 of \cite{Varshalovich1988}
\begin{equation}\label{eq: paper_43}
\left(\vec{Y}_{l_1, l_1, m_1}\times \vec{Y}_{l_2, l_2, m_2}\right) = i\sqrt{\frac{3}{2\pi}}{(2l_1+1)(2l_2+1)}\left[\sum_{j,l} \begin{Bmatrix}
l_1 & l_1 & 1\\
l_2 & l_2 & 1\\
j & l & 1
\end{Bmatrix}C^{l 0}_{l_1 0, l_20}C^{jm}_{l_1m_1, l_2m_2} \vec{Y}_{j,l,m}\right]
\end{equation}

where the curly bracket terms $\{ \}$ are the Wigner-9j symbols. To find the integral, we will use orthogonality of the vector harmonics, and so we first find it useful to re-write $\hat{r} Y^{l_3}_{m_3}(\hat{r})$ in terms of the irreducible vector harmonics. From Equation 70, page 219 of \cite{Varshalovich1988}, we have that 
\begin{equation}\label{eq: identity}
\hat{r} \,Y^{l_3}_{m_3}(\hat{r})
=
\sqrt{\frac{l_3}{2l_3+1}}\,\vec{Y}_{ l_3,l_3-1,m_3}(\hat{r})
-
\sqrt{\frac{l_3+1}{2l_3+1}}\,\vec{Y}_{ l_3,l_3+1,m_3}(\hat{r})\,.
\end{equation}

Using the orthonormality of the tensor harmonics in Equation \eqref{eq: orthonormality} we now subsitute Equation \eqref{eq: identity} into Equation \eqref{eq: paper_43} and integrate over the sphere to find
\begin{equation}{\label{eq: res_app}}
\begin{split}
    & \int_{\mathbb{S}^2} \left(\vec{Y}_{l_1,l_1,m_1}\times \vec{Y}_{l_2,l_2,m_2}\right) \cdot \left(\hat{r} Y_{m_3 }^{\ell_3}\right)^*d\mu_{\mathbb{S}^2}(\hat{r}) \\
    & = i \sqrt{\frac{3}{2 \pi}} \frac{( 2 l_1 +1)( 2 l_2 + 1)}{\sqrt{2 l_3 + 1}}
\left[\sqrt{l_3}C^{(l_3-1)0}_{l_10,l_20}\begin{Bmatrix}
l_1 & l_1 & 1 \\
l_2 & l_2 & 1 \\
l_3 & l_3 -1 & 1
\end{Bmatrix} - \sqrt{l_3 + 1}C^{(l_3+1)0}_{l_10,l_20}\begin{Bmatrix}
l_1 & l_1 & 1 \\
l_2 & l_2 & 1 \\
l_3 & l_3 +1 & 1
\end{Bmatrix}\right]
C^{l_3m_3}_{l_1m_1,l_2m_2},
    \end{split}
\end{equation}
Recalling Equation \eqref{eq: as_gaunt}, Equation \eqref{eq: res_app} implicitly defines $\Lambda_{l_1 l_2 l_3}$ in Equation \eqref{eq: as_gaunt}
\begin{equation}\label{eq: lambda}
\begin{split}
\Lambda_{l_1 l_2 l_3}=&    -i \sqrt{\frac{3}{2 \pi}\frac{l_1 l_2 (l_1+1)(l_2+1)}{(2 l_3 +1)}} ( 2 l_1 +1)( 2 l_2 + 1) \\
& \times
\left[\sqrt{l_3}C^{(l_3-1)0}_{l_10,l_20}\begin{Bmatrix}
l_1 & l_1 & 1 \\
l_2 & l_2 & 1 \\
l_3 & l_3-1 & 1
\end{Bmatrix} - \sqrt{l_3 + 1}C^{(l_3+1)0}_{l_10,l_20}\begin{Bmatrix}
l_1 & l_1 & 1 \\
l_2 & l_2 & 1 \\
l_3 & l_3+1 & 1
\end{Bmatrix} \right].
\end{split}
\end{equation}
It remains to study the selection rules of $\Lambda_{l_1 l_2 l_3}$ and to show that its non-vanishing for odd parity $l_1 + l_2 - l_3 = 2k +1$. 

To show this, it is sufficient to study the components with the two Wigner-9j terms. Our method will be to first show that $C^{(l_3-1)0}_{l_10,l_20}$ and $C^{(l_3+1)0}_{l_10,l_20}$ are always strictly non-vanishing for odd parity and have opposite signs. Our proof then follows by showing the Wigner-9j symbols appearing in Equation \eqref{eq: as_gaunt} always have the same sign and never vanishes simultaneously.  

The explicit form of the Clebsch-Gordan coefficients for $m_1 = m_2 = m_3 =0$ can be read off from Equation 32 page 251 of \cite{Varshalovich1988}
\begin{equation}
C^{l_3 0}_{l_1 0, l_2 0}
=
\begin{cases}
0, & \text{if } l_1 + l_2 + l_3 = 2g + 1, \\[6pt]
\dfrac{(-1)^{g-l_3}\sqrt{2l_3+1}\, g!}
{(g-l_1)!(g-l_2)!(g-l_3)!}
\left[
\dfrac{(2g-2l_1)!(2g-2l_2)!(2g-2l_3)!}{(2g+1)!}
\right]^{\frac12},
& \text{if } l_1 + l_2 + l_3 = 2g .
\end{cases}
\end{equation}
We immediately see that $\Lambda_{l_1 l_2 l_3}$ vanishes unless
\begin{equation}
    l_1 + l_2 + l_3 \pm 1 = 2g,
\end{equation}
which is equivalent to the parity condition, given $l_1, l_2, l_3$ are integers. The sign of $C^{l_3 0}_{l_1 0, l_2 0}$ is thus determined by $(-1)^{g-l_3}$. We immediately see that
\begin{equation}
\begin{split}
    \text{Sign}\left( C^{(l_3-1) 0}_{l_1 0, l_2 0}\right) & = (-1)^{\frac{l_1 + l_2 -l_3 +1 }{2} }\\
  \text{Sign}\left( C^{(l_3+1) 0}_{l_1 0, l_2 0}\right) & = (-1)^{\frac{l_1 + l_2 -l_3 -1 }{2} }\\
 \text{Sign}\left(\frac{C^{(l_3-1) 0}_{l_1 0, l_2 0}}{ C^{(l_3+1) 0}_{l_1 0, l_2 0}}\right) & = (-1)^{\frac{1}{2} + \frac{1}{2}} =-1.
    \end{split}
\end{equation}
Since the $m=0$ Clebsch-Gordan coefficients have opposite signs, it remains to show that the Wigner-9j coefficients have the same sign and cannot be simultaneously vanishing. This fact follows from  Equations for the form of Wigner $9j$ symbols of the form in Table 10.9 page 383 of \cite{Varshalovich1988}. By taking the transpose of the formula, and then swapping the first two columns, we find: 

\begin{equation}
\label{eq: wigner-9j-c+1}
\left\{
\begin{matrix}
a & a & 1 \\
b & b & 1 \\
c & c+1 & 1
\end{matrix}
\right\}
=
-2(c+1)
\left[
\frac{(a+b+c+2)(a+b-c)(a-b+c+1)(-a+b+c+1)(2a-1)!(2b-1)!(2c)!}
{3(2a+2)!(2b+2)!(2c+3)!}
\right]^{1/2}.
\end{equation}

\begin{equation}
\label{eq: wigner-9j-c-1}
\left\{
\begin{matrix}
a & a & 1 \\
b & b & 1 \\
c & c-1 & 1
\end{matrix}
\right\}
=
-2c
\left[
\frac{(a+b+c+1)(a+b-c+1)(a-b+c)(-a+b+c)(2a-1)!(2b-1)!(2c-2)!}
{(2a+2)!(2b+2)!(2c+1)!}
\right]^{1/2}.
\end{equation}
It is clear that both terms have the same sign. Then, if any of the terms $(a+b-c), (a-b+c), (-a+b+c)$ is zero,  $a+b+c$ must be even. By contraposition, none of these terms can vanish whenever $a+b+c$ is odd. As a consequence, assuming $a+b+c$ is odd, the Wigner-9j symbol of Eq.~\eqref{eq: wigner-9j-c-1} is equal to zero if and only if $a+b-c+1=0$, which yields $a+b=c-1$. In this case, we get $a-b+c+1 =a-b+a+b+2=2(a+1)>0$ and $-a+b+c+1 = -a+b+a+b+2=2(b+1)>0$, such that the Wigner-9j symbol of Eq.~\eqref{eq: wigner-9j-c+1} is non-zero. Reciprocally, assuming $a+b+c$ is odd, the Wigner-9j symbol of Eq.~\eqref{eq: wigner-9j-c+1} is equal to zero if and only if $a-b+c+1=0$ or $-a+b+c+1=0$. If $a-b+c+1=0$ we have that $b=a+c+1$, and it comes $a+b-c+1=a-c+1+a+c+1=2(a+1)>0$. Symmetrically, if $-a+b+c+1=0$ we get $ a = b+c+1$ such that $a+b-c+1 = 2(b+1)>0$. As a result, the Wigner-9j symbol of Eq.~\eqref{eq: wigner-9j-c-1} is non-zero whenever the one of Eq.~\eqref{eq: wigner-9j-c+1} is. This concludes our proof that Equation \eqref{eq: as_gaunt} holds, where $\Lambda_{l_1 l_2 l_3}$ (as defined by Equation \eqref{eq: lambda}) is non-vanishing for odd parity $l_1 + l_2 + l_3 = 2k+1$, and zero otherwise. 

Equation \eqref{eq: as_gaunt} is a tensor equation that holds in the complex spherical harmonic basis. In Section \ref{sec: transformation} we explicitly transform it to the real spherical harmonics basis. With this choice of basis it is possible to show that the rotated Clebsch-Gordan coefficients for odd parity are purely imaginary, which explains the factor of $i$ appearing in Equation \eqref{eq: lambda}. 

\subsection{Transforming from complex to real harmonics}\label{sec: transformation}
In this section we show explicitly how Equation \eqref{eq: as_gaunt} transforms as we change from the complex spherical basis to the real Cartesian basis, and from complex spherical harmonics to the real spherical harmonics. 

Let $\{e_x,e_y,e_z\}$ denote the Cartesian basis and
$\{e_{+1},e_0,e_{-1}\}$ the complex spherical basis
\begin{equation}\label{eq:sph_vec_basis}
e_{\pm1}=\mp\frac{1}{\sqrt2}(e_x\pm i e_y),\qquad e_0=e_z.
\end{equation}
The change of basis is described by the unitary matrix
\begin{equation}
W^{(1)}=
\begin{pmatrix}
-\frac{1}{\sqrt2} & 0 & \frac{1}{\sqrt2}\\
\frac{i}{\sqrt2} & 0 & \frac{i}{\sqrt2}\\
0 & 1 & 0
\end{pmatrix}.
\end{equation}

There are two separate rotations we can do on Equation \eqref{eq: as_gaunt}. Firstly, we can rotate the tangent space vectors with $W^{(1)}$. These vectors arise from the $m_s$ spin indices of Equation \eqref{eq: tensor_harmonic}. Separately, we can also choose a different basis for the spherical harmonics themselves via a rotation on the $m$ indices. The fact that we have both these choices arises because the vector harmonics form a basis for irreducible representations of vector fields on the sphere. 

Since Equation \eqref{eq: as_gaunt} is a scalar equation in the tangent space, we know that it transforms as a scalar under changes of basis of the tangent space, and so the change of basis from spherical to Cartesian coordinates on the tangent space is trivial. However, Equation \eqref{eq: as_gaunt} has free $m$ labels, and so the equation does transform when we change the basis from complex to real harmonics. 

Specifically, let $W^{(l)}$ be a unitary change-of base matrix from complex spherical basis to the real spherical basis $Y_{lm}$
\begin{equation}
Y_{lm}
=
\sum_{m} W^{(l)}_{m m'}\,Y^{l}_{m'} .
\end{equation}
$W^{(l)}$ can be defined by the following relations 
\begin{equation}
\begin{split}
& Y_{l m} = 
\frac{(-1)^m}{\sqrt{2}}
\left(
Y^{l}_{m}
+
(-1)^m Y^{l}_{-m}
\right), \ \
Y_{l,-m}
=
\frac{1}{\sqrt{2}\,i}
\left(
Y^{l}_{m}
-
(-1)^m Y^{l}_{-m}
\right),  \ \ m >0\\
& Y_{l0}=Y^{l}_{0}.
\end{split}
\end{equation}
Under such a rotation, both the left and right hand side of Equation \eqref{eq: as_gaunt} transform. In particular we find 
\begin{equation}\label{eq: as_gaunt_real}
    \int_{\mathbb{S}^{2}}(\nabla Y_{l_1 m_1} \times \nabla Y_{l_2 m_2})\cdot (\hat{r} Y_{l_3m_3}) \mathrm{d}\mu_{\mathbb{S}^{2}}(\hat{r}) = \Lambda_{l_1 l_2 l_3}\left(C^{\mathbb R}\right)^{l_3 m_3}_{l_1 m_1\,l_2 m_2},
\end{equation}

where we have defined the real basis Clebsch-Gordan tensor
\begin{equation}\label{eq: real_cg}
\left(C^{\mathbb R}\right)^{l_3 m_3}_{l_1 m_1\,l_2 m_2}
=
\sum_{m'_1,m'_2,m'_3}
W^{(l_1)}_{m_1 m'_1}\,
W^{(l_2)}_{m_2 m'_2}\,C^{l_3 m'_3}_{l_1 m'_1\,l_2 m'_2}\,\left(W^{(l_3)\dagger}\right)_{m'_3 m_3}\,.
\end{equation}
Note that the left hand side of Equation \eqref{eq: as_gaunt_real} is now real, and so the right hand side must also be. However, there is a puzzling factor in $i$ in Equation \eqref{eq: lambda} for $\Lambda_{l_1, l_2, l_3}$. This is accounted for when we note that the real-basis Clebsch-Gordan tensor is purely imaginary for odd parity. More generally we have 
\begin{equation}\label{eq: real_parity}
\left(\left(C^{\mathbb R}\right)^{l_3 m_3}_{l_1 m_1\,l_2 m_2}\right)^{*}
=
(-1)^{l_1+l_2-l_3}
(C^{\mathbb R})^{l_3 m_3}_{l_1 m_1\,l_2 m_2}.
\end{equation}
so
\begin{equation}
(C^{\mathbb R})\in
\begin{cases}
\mathbb{R} & l_1+l_2-l_3 \text{ even},\\
i\,\mathbb{R} & l_1+l_2-l_3 \text{ odd}.
\end{cases}
\end{equation}

We now give a short proof of Equation \eqref{eq: real_parity}.

Consider, the complex conjugate of Equation \eqref{eq: real_cg}
\begin{align}\label{eq: conj_real}
\big((C^{\mathbb R})^{l_3 m_3}_{l_1 m_1\,l_2 m_2}\big)^*
&=
\sum_{m'_1,m'_2,m'_3}
W^{(l_1)*}_{m_1 m'_1}\,
W^{(l_2)*}_{m_2 m'_2}\,C^{l_3 m'_3}_{l_1 m'_1\,l_2 m'_2}\,W^{(l_3)T}_{m'_3 m_3}\,.
\end{align}

For the complex to real spherical harmonic basis change, the matrix $W^{(l)}$ satisfies
\begin{equation}
W^{(l)*}_{m m'}
=
(-1)^{m'} W^{(l)}_{m,-m'},
\end{equation}
which follows from the identity $Y^{l*}_{m}=(-1)^m Y^{l}_{-m}$. We now use the known identity (see Equation 24, page 247 \cite{Varshalovich1988}) that
\begin{equation}\label{eq: parity_identity}
C^{l_3 m_3}_{l_1 m_1\,l_2 m_2} = 
(-1)^{l_1+l_2-l_3}
C^{l_3,-m_3}_{l_1,-m_1\,l_2,-m_2}
\end{equation}
Substituting Equation \eqref{eq: parity_identity} into Equation \eqref{eq: conj_real} and relabeling
the summation indices we find
\begin{equation}
\big((C^{\mathbb R})^{l_3 m_3}_{l_1 m_1\,l_2 m_2}\big)^* =(-1)^{l_1+l_2-l_3}
\sum_{m'_1,m'_2,m'_3}
W^{(l_1)}_{m_1m'_1}\,
W^{(l_2)}_{m_2m'_2}\,
C^{l_3 m'_3}_{l_1 m'_1\,l_2 m'_2}\,
\left(W^{(l_3)\dagger}\right)_{m'_3 m_3}\,,
\end{equation}
where we have also used that $(-1)^{m'_1 + m'_2 + m'_3} = (-1)^{2m'_3} =1$ via selection rules. Recognizing the original definition of $(C^{\mathbb R})$ in Equation \eqref{eq: real_cg} we find
\begin{equation}
\big((C^{\mathbb R})^{l_3 m_3}_{l_1 m_1\,l_2 m_2}\big)^*
=
(-1)^{l_1+l_2-l_3}
(C^{\mathbb R})^{l_3 m_3}_{l_1 m_1\,l_2 m_2}.
\end{equation}
which completes the proof. 

Note that in computational libraries such as e3nn \cite{geiger2022e3nneuclideanneuralnetworks}, the rotation of the Clebsch-Gordan coefficients is defined with an extra factor of $i$ in according to $\tilde{W}^{(l)} = (-i)^l W^{(l)}$. With this in mind, the real e3nn Clebsch-Gordan coefficients are defined by
\begin{equation}
    (\tilde{C}^{\mathbb{R}})_{l_1 m_1 l_2 m_2}^{l_3 m_3} = i^{l_1 + l_2  -l_3}(C^{\mathbb{R}})_{l_1 m_1 l_2 m_2}^{l_3 m_3} = i^{l_1 + l_2  +l_3}(-1)^{l_3}(C^{\mathbb{R}})_{l_1 m_1 l_2 m_2}^{l_3 m_3} 
\end{equation}
which is always real. With these conventions, our result Eq.~\eqref{eq: as_gaunt_real} reads 
\begin{equation}
    \int_{\mathbb{S}^{2}}(\nabla Y_{l_1 m_1} \times \nabla Y_{l_2 m_2})\cdot (\hat{r} Y_{l_3m_3}) \mathrm{d}\mu_{\mathbb{S}^{2}}(\hat{r}) = \text{Im}(\Lambda_{l_1 l_2 l_3}) \ (-i)^{l_1 + l_2  +l_3 -1 }(-1)^{l_3}(\tilde{C}^{\mathbb R})^{l_3 m_3}_{l_1 m_1\,l_2 m_2},
\end{equation}
which is always real for odd parity. Finally, identifying the above formula with Eq.~\eqref{eq: main_result_odd} of the main text, we get
\bea
\tilde{V}^{l_3}_{l_1,l_2} =  \ (-1)^{\frac{l_1 + l_2 + l_3 -1}{2}+l_3}\text{Im}(\Lambda_{l_1 l_2 l_3})\,.
\eea

\subsection{Low-Rank Decomposition of Coupling Coefficients}
\label{app:lowrank}

In this Section we discuss the normalization of the antisymmetric Gaunt
coupling coefficients introduced in the main text. When these quantities
are used within $O(3)$-equivariant neural networks, it is desirable that
the associated tensor products have magnitudes comparable to standard
Clebsch-Gordan tensor products so that network layers are well
initialized and numerically stable.

As discussed in Section \ref{sec: expressivity}, the efficient implementation of an integral tensor product formula relies on a factorized representation of the spherical integral that separates the dependence on $l_1$, $l_2$, and $l_3$. If the tensor product is explicitly normalized by multiplying each channel by a general tensor, this factorization would be destroyed unless the normalization itself can be expressed in a separable form.

In particular, practical use of the factorized integral formula discussed in the main text benefits from an explicit low-rank decomposition of the coupling tensor if one wants to explicitly normalize the tensor product. We therefore wish to approximate normalization factors of the form 
\begin{equation}\label{eq: decomp-app}
  (T_{l_1 l_2}^{l_3})^{-1}
  \;\approx\;
  \sum_{r=1}^{R} a_{l_1,r}\, b_{l_2,r}\, c_{l_3,r}
\end{equation}
for a small $R$, so that we can initialize layers to be approximately the same magnitude as the Clebsh-Gordan tensor product and benefit from the factorization trick. In Equation \eqref{eq: decomp-app}, $T_{l_1, l_2}^{l_3}$ can be either the anti-symmetric or symmetric Gaunt coefficients, $T_{l_1, l_2}^{l_3} \in \{ \tilde{V}_{l_1, l_2}^{l_3}, \tilde{G}_{l_1, l_2}^{l_3} \}$. 

Empirically we find that $(T_{l_1 l_2}^{l_3})^{-1}$ is intrinsically
low rank. Figure~\ref{fig:scaling} compares the quality of rank-1 and rank-2 approximations across $\lmax$ for $T= \tilde{V}$, the antisymmetric Gaunt coefficients. A rank-$2$ decomposition accurately reproduces the tensor over a wide range of angular momenta. In particular, a rank-$2$ approximation achieves approximately
$10\%$ per-entry accuracy for all tested values up to $\lmax = 19$.
By contrast, a rank-$1$ approximation fails qualitatively and cannot
reproduce the structure of the tensor. Rank~1 is qualitatively incorrect ($\logsig \sim 1$--$2$, and we also noticed sign errors over large subspaces), whereas rank~2 achieves $\logsig < 0.04$ at all tested $\lmax$, with 100\% of entries accurate to within a factor of two. This behavior likely arises from
the two distinct coupling pathways $l_3 \pm 1$ appearing in
Eq.~\eqref{eq: lambda}, which introduce two dominant modes in the
coupling tensor.

Figure~\ref{fig:comparison} compares fits of $(\tilde{V}^{l_3}_{l_1 l_2})^{-1}$ with fits of the even-parity Gaunt coefficients $(\Gtil^{l_3}_{l_1,l_2})^{-1}$.
The Gaunt tensor depends on a single Clebsch--Gordan coefficient coupling, and appears approximately rank~1 ($\logsig < 0.13$). This perhaps explains why the authors of \cite{luo2024enabling} did not encounter issues with normalization.

The decomposition in Equation by minimizing the mean relative squared error
\begin{equation}
  \mathcal{L} =
  \frac{1}{N}
  \sum_{i=1}^{N}
  \frac{(y_i - \hat{y}_i)^2}{y_i^2},
  \label{eq:loss}
\end{equation}
where the sum runs over the $N$ non-zero entries permitted by the
selection rules. This loss treats all tensor elements equally on a
relative scale, which is important because the
$\tilde{V}_{l_1 l_2}^ {l_3}$ span several orders of magnitude. To quantify the quality of the approximation we report the logarithmic
ratio spread
\begin{equation}
    \logsig = \operatorname{std}\!\left[\log_{10}|y_i/\hat{y}_i|\right]
\end{equation}
which measures the dispersion of multiplicative errors. In this metric,
$\logsig = 0$ corresponds to exact agreement,
$\logsig = 0.3$ corresponds to a factor-of-two scatter,
and $\logsig = 0.04$ corresponds to approximately $10\%$ relative error.

Optimization was performed using L-BFGS-B with numerically estimated
gradients, convergence tolerances
$f_{\mathrm{tol}} = 10^{-15}$ and $g_{\mathrm{tol}} = 10^{-10}$,
and a maximum of 300 iterations per run. In practice the optimization
converged reliably, and alternative optimization methods yielded
similar results.

\begin{figure}[htbp]
  \centering
  \includegraphics[width=\textwidth]{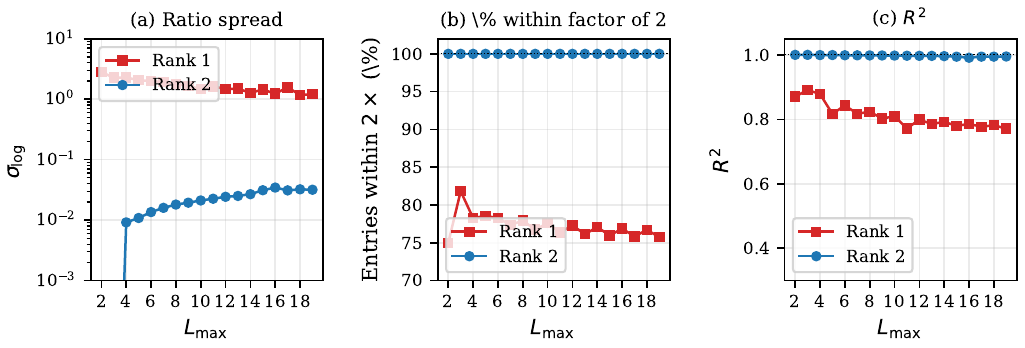}
  \caption{%
    Decomposition quality for $(\tilde{V}_{l_1 l_2}^{l_3})^{-1}$ vs $\lmax$.
    \textbf{(a)}~$\logsig$ (log scale): rank~2 remains below 0.04,
    while rank~1 exhibits order-of-magnitude errors.
    \textbf{(b)}~100\% of rank-2 entries lie within a factor of $2$.
    \textbf{(c)}~$R^2$ for the rank-2 fit remains above $0.9$.
  }
  \label{fig:scaling}
\end{figure}

\begin{figure}[htbp]
  \centering
  \includegraphics[width=\textwidth]{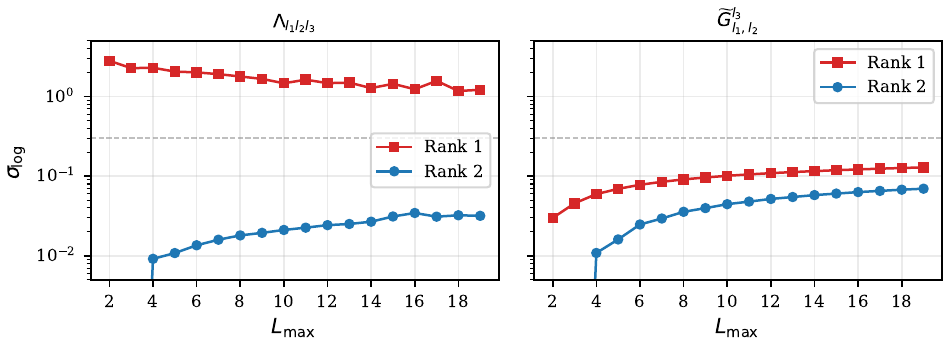}
  \caption{%
    $\logsig$ vs $\lmax$ for $(\tilde{V}^{l_3}_{l_1 l_2})^{-1}$ (left)
    and $(\Gtil^{l_3}_{l_1,l_2})^{-1}$ (right).
    The dashed line indicates $2\times$ scatter. We see that 
    $\tilde{V}^{-1}$ requires rank~2, whereas $\Gtil^{-1}$ a single rank approximation is sufficient for normalization in a neural network. 
  }
  \label{fig:comparison}
\end{figure}

\end{document}